\documentclass[12pt,journal,draftcls,letterpaper,onecolumn]{IEEEtran}

\usepackage{amsmath,amssymb,amsthm,graphicx,cite,algpseudocode,algorithm,caption,subcaption,setspace}

\newtheorem{theorem}{\bf Theorem}[section]

\newtheorem{question}{\bf Question}[section]
\newtheorem{definition}{\bf Definition}[section]
\newtheorem{proposition}{\bf Proposition}[section]
\newtheorem{remark}{\bf Remark}[section]
\newtheorem{lemma}{\bf Lemma}[section]

\algnewcommand{\Inputs}[1]{%
  \State \textbf{Inputs:}
  \Statex \hspace*{\algorithmicindent}\parbox[t]{.8\linewidth}{\raggedright #1}
}

\algnewcommand{\Initialize}[1]{%
  \State \textbf{Initialize:}
  \Statex \hspace*{\algorithmicindent}\parbox[t]{.8\linewidth}{\raggedright #1}
}

\begin{document}
\title{Short and Wide Network Paths}

\author{Lavanya Marla$^1$, Lav~R.~Varshney$^1$, Devavrat Shah $^2$, Nirmal A. Prakash$^1$, and Michael E. Gale$^1$
\thanks{$^1$ University of Illinois at Urbana-Champaign}
\thanks{$^2$ Massachusetts Institute of Technology}
}

\maketitle

\begin{abstract}
Network flow is a powerful mathematical framework to systematically explore the relationship between structure and function in biological, social, and technological networks.
We introduce a new pipelining model of flow through networks where commodities must be transported over single paths rather than split over several paths and recombined.  We show this notion of pipelined network flow is optimized using network paths that are both short and wide, and develop efficient algorithms to compute such paths for given pairs of nodes and for all-pairs.  Short and wide paths are characterized for many real-world networks.

To further demonstrate the utility of this network characterization, we develop novel information-theoretic lower bounds on computation speed in nervous systems due to limitations from anatomical connectivity and physical noise.
For the nematode \emph{Caenorhabditis elegans}, we find these bounds are predictive of biological timescales of behavior.  Further, we find the particular \emph{C.~elegans} connectome is  globally less efficient for information flow than random networks, but the hub-and-spoke architecture of functional subcircuits is optimal under constraint on number of synapses.  This suggests functional subcircuits are a primary organizational principle of this small invertebrate nervous system.
\end{abstract}

\section{Introduction}
\label{sec:intro}
In studying complex systems via the interconnection of their elements, network science has emerged in the last two decades as an insightful approach for understanding collective behavior in brains, societies, and physical infrastructures.  Common network science analysis techniques draw on dynamical systems theory \cite{WattsS1998, Newman2010} and many universal properties of disparate networks have been found \cite{BarabasiA1999}.  Other prevalent analysis techniques are based on network flow \cite{AhujaMO1993}. 

We take the flow perspective and introduce a novel notion of network flow that arises in many biological, social, and technological networks, yet has not previously been studied. Consider a network where a commodity to be transmitted from a source node to a destination node can be split into pieces in time and sent in a pipelined fashion over (possibly) several hops using as many time slots as needed.  The commodity, however, must go over a single route rather than being split over several routes to be recombined by the destination \cite{Pollack1960,Hu1961}.  This is in contrast to the maximum capacity problem \cite{FordF1956,EliasFS1956} where  
flow between two nodes may use as many different routes as needed.  This ``circuit-switched'' model with a single route is prevalent in systems without the ability to split and recombine, e.g.\ signal flow in simple neuronal networks, message flow in social networks, and the flow of train cars in railroad networks with few engines.   

We will see that the optimal paths for the pipelined network flow problem must not only be short in terms of number of hops but also wide in terms of the bottleneck edge in the route.  That is, to maximize flow requires finding the single best route between the two nodes: the route that minimizes the weight of the maximum-weight edge in the route and yet is short in path length.  Finding all-pairs shortest paths in weighted networks can be accomplished in polynomial time using Floyd's dynamic programming algorithm \cite{Floyd1962}.  This is an optimization problem in a metric space. Finding widest paths in weighted networks can be accomplished by taking paths in a maximum spanning tree \cite{Hu1961,Pollack1960}, also in polynomial time.  This is an optimization problem in an ultrametric (non-Archimedean) space \cite{RammalTV1986}. As far as we can tell, our problem of finding paths that minimize the width-length product between two nodes has remained unstudied in the literature. We develop efficient algorithms for finding short and wide paths between two given nodes, as well as for all pairs.  As part of the development, we prove correctness and also characterize the computational complexity as polynomial time.  Depth-first search strategies that enumerate all simple paths between two nodes \cite{Black2008}, would check a factorial number of paths in the worst case.  

Efficient algorithms enable us to characterize the all-pairs distribution of short and wide paths for many complex networks.  Note that traditional notions of network diameter and  average path length are studied extensively in network science \cite{Newman2010}, but the all-pairs geodesic distance distribution of unweighted graphs of fairly arbitrary topology is also starting to be 
of interest \cite{BauckhageKH2015,MelnikG2016_arXiv}, essentially building on results for Erd\"{o}s-R\'{e}nyi random graphs \cite{BlondelGHJ2007,KatzavNBKKRB2015}.  As far as we know, this distance distribution for weighted graphs remains unstudied, as does the distribution of our short-and-wide path lengths.  In studying the short-and-wide path length distribution, we do not see universality across networks.

To demonstrate the detailed structure-function insights that pipelined flow gives, we consider neuronal networks.  Indeed with advances in experimental connectomics producing wiring diagrams of many neuronal networks, there is growing interest in informational systems theories to provide insight \cite{SpornsTK2005,Seung2011,Seung2012}.  For concreteness, we focus on hermaphrodite nematode \emph{Caenorhabditis elegans}, which 
is a standard model organism in biology \cite{Rockland1989,BonoM2005,SenguptaS2009} and has exactly $302$ neurons \cite{VarshneyCPHC2011}.
We consider three scientific questions in asking whether information transmission through the nervous system is a bottleneck that limits behavior.  (Neural efficiency hypotheses 
of intelligence also argue information flows better in the nervous systems of bright individuals 
\cite{DearyPJ2010, LeeWYWC2012}.)  
\begin{question}
Do neuronal circuits allow behaviors to happen as quickly as possible under
information flow limitations imposed by synaptic noise properties and neuronal connectivity patterns?
\end{question}
\begin{question}
Are information flow properties of neuronal networks significantly different from random graphs drawn from ensembles that match other network functionals?  That is, are networks non-random \cite{SongSRNC2005} in allowing information flow that is faster or slower than other networks?  
\end{question}
\begin{question}
Does the synaptic microarchitecture of functional subcircuits optimize information flow under constraint on number of synapses?
\end{question}

Since the exact computations performed by the nervous system are unclear, we use
general information-theoretic methods to lower bound optimal computational performance 
of a given neural circuit in terms of its physical noise and connectivity structure.  This approach is inspired by information-theoretic limits in distributed computing and control \cite{MartinsD2008,AyasoSD2010}.  If the performance of a neural circuit is close to the lower bound, then it is operating close to optimally.

We specifically consider gap junctions in \emph{C.\ elegans}, where neurons are directly
electrically connected to each other through pores in their membranes.  There can be more than one gap junction connecting two given neurons.  We, for the first time, model and compute the Shannon channel capacity of gap junctions.  Channel capacity used together with the network topology of the system in the short-and-wide path computation, and with an estimate of the informational requirements to perform computations, yields a bound on the minimum time to perform biologically plausible computations. Remarkably, when this lower bound is applied to \emph{C.\ elegans}, the result is rather close to behaviorally-observed timescales.  This suggests nematodes may be operating close to the behavioral limits imposed by physical properties of their nervous system.

In asking whether the network is non-random \cite{SongSRNC2005} in allowing
behavior that is faster or slower than other networks, we surprisingly find that the 
complete \emph{C.~elegans} connectome has greater distances and is therefore slower in supporting global information flow than random networks. Contrarily, we prove that the hub-and-spoke architecture of \emph{C.\ elegans} functional subcircuits
\cite{MacoskoPFCBCB2009,RabinowitchCS2013} optimizes computation speed under constraint on number of synapses.  As such, global information flow may not be a relevant criterion for neurobiology, at least for a small invertebrate system like \emph{C.~elegans}.  Rather, functional subcircuits may be a primary organizational principle.

\section{Pipelining Model of Information Flow}
Consider a network where a commodity is to be sent from a source to a destination in a manner that can be split into pieces in time and sent in a pipelined fashion over (possibly) several
hops using as many time slots as needed.  The commodity must go over a single route  
rather than being split over several routes to be recombined by the destination.  As an example, in \emph{C.\ elegans}, each neuron is identified by name and is different from any other neuron \cite{VarshneyCPHC2011}; computational specialization may arise from neuronal
specialization, which in turn may require specific paths for specific information.  

In this model, maximizing flow requires finding the single best route between the two nodes:
the route that minimizes the weight of the maximum-weight edge in the route and yet is short in path length.  In the context of network behavior, note that since we adopt the short-and-wide path view of information flow rather than the maximum capacity view \cite{FordF1956,EliasFS1956}, bounds on computation speed
will be governed by an appropriate notion of graph diameter rather than by notions of graph conductance \cite{AyasoSD2010}.
Since diameter provides weaker bounds than graph conductance, this is without loss of
generality.  Our notion of graph diameter is defined in the next section.

\subsection{Distance and Effective Diameter}\label{subsec:distance_definitions}
Consider the following standard definitions of graph distance for undirected, weighted graphs. 
\begin{definition} 
Let $G = (V,E)$ be a weighted graph.  Then the \emph{geodesic distance}
between nodes $s,t \in V$ is denoted $d_G(s,t)$ and is the number of edges connecting
$s$ and $t$ in the path with the smallest number of hops between them.  If there is no path connecting 
the two nodes, then $d_G(s,t) = \infty$.
\end{definition}

\begin{definition} 
Let $G = (V,E)$ be a weighted graph.  Then the \emph{weighted distance}
between nodes $s,t \in V$ is denoted $d_W(s,t)$ and is the total weight of edges connecting
$s$ and $t$, in the path with the smallest total weight between them.  If there is no path connecting 
the two nodes, then $d_W(s,t) = \infty$.
\end{definition}

Another notion of distance arises from the pipelining model of flow.  We want a path between two nodes that has a small number of hops but is also such that the 
weight of the maximum-weight edge is small; we measure path length weighted by this bottleneck weight. 
\begin{definition} 
Let $G = (V,E)$ be a weighted graph.  Then the \emph{bottleneck distance}
between nodes $s,t \in V$ is denoted $d_B(s,t)$ and is the number of edges connecting $s$ and $t$,
scaled by the weight of the maximum-weight edge, in the path with the smallest
total scaled weight between them.  If there is no path connecting 
the two nodes, then $d_W(s,t) = \infty$.
\end{definition}

\begin{proposition}
\label{prop:order1}
If weights of all actual edges are $1$ or less, geodesic distance upper bounds the bottleneck distance:
\begin{equation}
d_B(s,t) \le d_G(s,t)\mbox{.}
\end{equation}
\end{proposition}
\begin{proof}
Consider the path between $s$ and $t$ that governs $d_G(s,t)$.  Since the weights of all actual edges are $1$ 
or less, the maximum-weight edge weight is $1$ or less.  Hence the bottleneck weight of this path must be less
than or equal to the number of edges connecting $s$ and $t$ in that path.  Since by definition $d_B(s,t)$ 
minimizes the bottleneck weight among paths connecting $s$ and $t$, $d_B(s,t) \le d_G(s,t)$.
\end{proof}
\begin{proposition}
\label{prop:order2}
Weighted distance lower bounds the bottleneck distance:
\begin{equation}
d_B(s,t) \ge d_W(s,t)\mbox{.}
\end{equation}
\end{proposition}
\begin{proof}
Consider the path between $s$ and $t$ that governs $d_B(s,t)$.  Let $w_0$ be the maximum-weight edge weight
and $h$ the number of edges in it.  Since other weights in the path are only less than $w_0$, the total weight of 
this path is less than $hw_0$, but by definition $d_W(s,t)$ is less than or equal to this quantity.
Hence $d_B(s,t) \ge d_W(s,t)$.
\end{proof}

It is convenient for the sequel to write these distances as constrained optimization problems.  We first define a set of constraints on the decision variables $x_{ij}$ that indicate how different edges are used in optimal paths, and $N$ indicates neighborhood.
\begin{align}
& \sum_{j:(s,j)\in E} x_{ij} = 1 \label{balance_1}\\
& \sum_{j:(i,j)\in E} x_{ij} - \sum_{j:(j,i)\in E} x_{ji} = 0 \: \mbox{ for all } \: i \in N\setminus{\{s,t\}}\label{balance_2}\\
& \sum_{i:(i,t)\in E} x_{it} = 1 \label{balance_3}\\
& x_{ij} \in \{0,1\} \quad \forall \: (i,j) \in E \label{balance_4}
\end{align}
These constraints simply enforce the movement of one unit of flow from node $s$ to node $t$ and maintain flow balance at all other nodes. Integrality constraints ensure that exactly one path is chosen. These constraints are the same as for standard shortest path problems \cite{AhujaMO1993} and therefore the constraint matrix is totally unimodular. The distance expressions use the notation $w_{ij}$ for the edge weight between nodes $i$ and $j$.

\begin{equation}
d_G(s,t) = \min \sum_{(i,j) \in E} x_{ij} \mbox{ such that \eqref{balance_1}--\eqref{balance_4} hold.} \label{obj:shortest_length} 
\end{equation}
\begin{equation}
d_W(s,t)  = \min \sum_{(i,j) \in E} w_{ij} x_{ij} 
 \mbox{ such that \eqref{balance_1}--\eqref{balance_4} hold.} \label{obj:min_weight}
\end{equation}
\begin{equation}
d_B(s,t) = \min \left\lbrace \max_{(i,j)\in E} \{w_{ij}x_{ij}\} \sum_{(i,j) \in E} x_{ij} \right\rbrace \mbox{ such that \eqref{balance_1}--\eqref{balance_4} hold.} \label{obj:min_bottleneck}
\end{equation}
Notice the objective functions are nonlinear; yet we develop efficient algorithms in Section \ref{sec:algorithms}.

Any of these distance functions can be used to define the all-pairs distance distribution of a network, which is just
the empirical distribution of distances among all $\binom{n}{2}$ pairs of vertices, for a graph of size $|V| = n$.  These distance distributions
have various moments and order statistics, such as the average path length and the diameter.

\begin{definition}
The graph \emph{diameter} is
\begin{equation}
D = \max_{s,t \in V} d(s,t) \mbox{.}
\end{equation}
\end{definition}
We also define a notion of effective diameter where node pairs that are outliers in the all-pairs distance distribution
do not enter into the calculation.  Recall that the quantile function corresponding to the cumulative
distribution function (cdf) $F(\cdot)$ is 
$
Q(p) = \inf \{x \in \mathbb{R} | p \le F(x) \}
$
for a probability value $0 < p < 1$.
\begin{definition}
For a network of size $n$, let $F(x)$ be the empirical 
cdf of the distances of all $\binom{n}{2}$ distinct node pairs.  
Then the \emph{effective diameter} is:
\begin{equation}
D_{e} = Q(0.95)\mbox{.}
\end{equation}
\end{definition}
This definition is more stringent than others in the literature \cite{LeskovecKF2007}.
Of course, $D_{e} \le D$.  Moving forward, we use effective diameter
rather than diameter since it characterizes when most of a commodity
would have reached its destination.  Thresholds other than $0.95$ can be easily defined.

\subsection{Unique Property of Bottleneck Paths}
\label{sec:features}
In this subsection, we discuss a property of short-and-wide paths that is distinct from geodesic and weighted paths and that has algorithmic importance.

A key attribute of shortest paths exploited by many shortest path algorithms is the \emph{optimal substructure property}, that all subpaths of shortest paths are shortest paths \cite{AhujaMO1993}. This follows from metric structure, but for the ultrametric space induced by short-and-wide paths, the property does not hold.  We prove this through a counterexample.

Consider the network in Figure \ref{fig:example}, with the network weights $w_{ij}$ shown on the edges. For source node $A$ and destination node $H$, the geodesic, weighted, and short-and-wide paths are, respectively: $A-E-G-H$ (3 units), $A-B-D-F-H$ (1.367 units), and $A-B-D-F-H$ (or $A-B-C-F-H$, 2 units). Instead, if we compute the same paths with $A$ as the source and $K$ as the destination, path $A-E-G-H-I-J-K$ is the geodesic path (6 units) $A-B-D-F-H-I-J-K$ is the weighted path (3.367 units) and path $A-E-G-H-I-J-K$ is the short-and-wide path. Notice if we compute the short-and-wide path from $A$ to $K$, this does not guarantee that its sub-path up to node $H$ is optimal for source $A$ and  destination $H$. 

This implies short-and-wide paths do not form trees, unlike geodesic or weighted paths. Hence, classic tree-based shortest path algorithms \cite{AhujaMO1993} must be modified significantly for this setting, as we now show.

\begin{figure}
  \centering
  \includegraphics[width=3.5in]{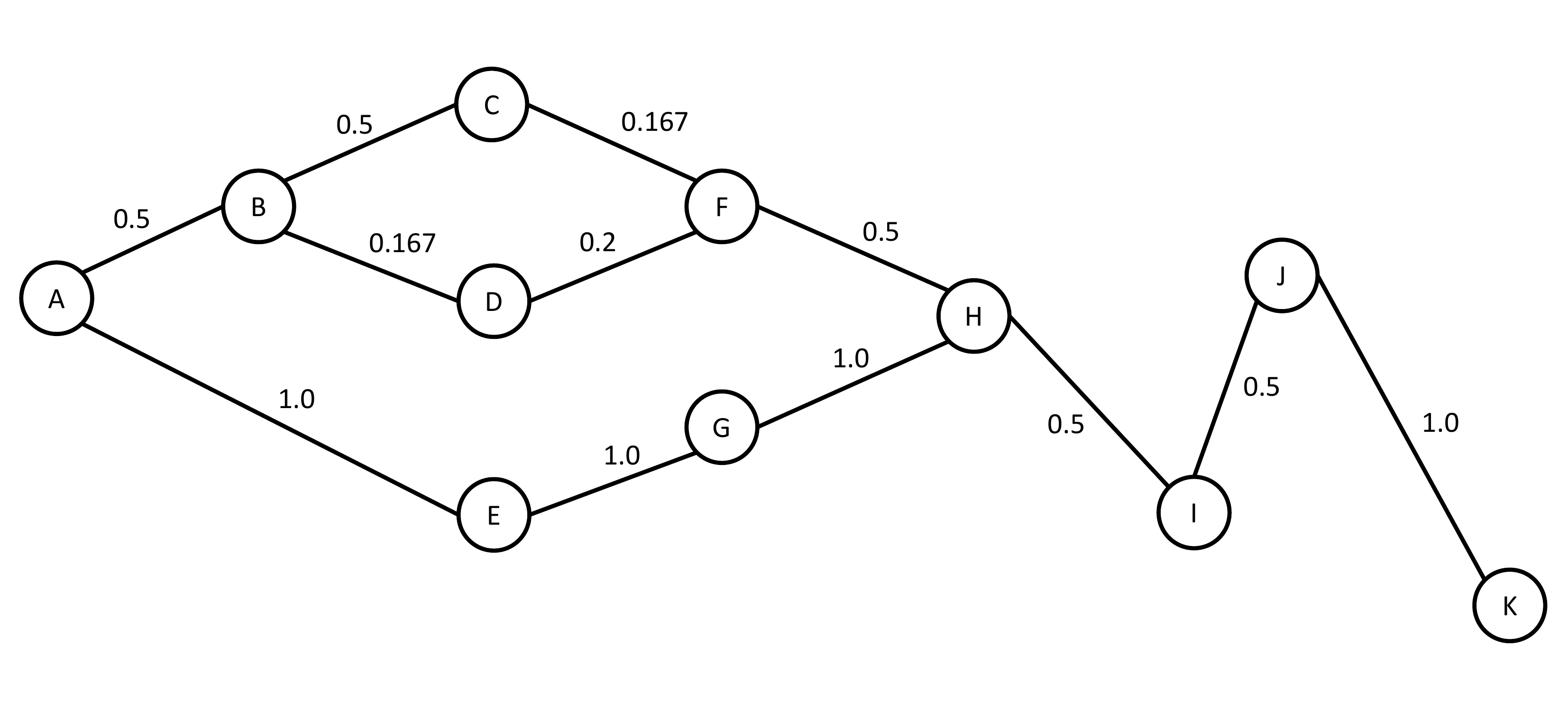}
  \caption{Network that demonstrates the optimal substructure violation for short and wide paths.}
  \label{fig:example}
\end{figure}

\section{Efficient Algorithms for Computing Short and Wide Paths}
\label{sec:algorithms}
In this section, we present algorithms for computing short-and-wide paths.

\subsection{One-to-All Bottleneck Distance Algorithm}

We present Algorithm \ref{bottleneck_dijkstra_one_all} to compute the bottleneck distance. The algorithm maintains label sets at each node $j \in V$, denoted by $L_p(j) = [p, l_p^1(j), l_p^2(j), l_p^3(j), pred_p(j), pp_p(j)]$, where the label sets include: (i) $p$, the index of the label set, (ii) $l_p^k(j)$, the value of label $k$ in label set $p$ at node $j$, $k$ = 1, 2, 3; where: (a) $l_p^1(j)$ tracks the number of edges traversed from $s$ until the current node, (b) $l_p^2(j)$ tracks the maximum width along a path from the origin until the current node, (c) $l_p^3(j) ( = l_p^1(j) \times l_p^2(j))$ tracks the product of the maximum width and the number of edges traversed from $s$ until the current node $j$; (iii) $pred_p(j)$: predecessor node for label set $p$ at node $j$, (iv) $pp_p(j)$: index of predecessor's label set for label $p$ at node $j$.  We also let  $np(j)$ denote the number of non-dominated label sets for node $j$. 

\begin{definition}\label{def:dominated_labels}
A label set $L_p(j)$ is \emph{strictly dominated} by label set $L_q(j)$ at node $j$, if $l_p^1(j) > l_q^1(j)$ and $l_p^2(j) > l_q^2(j)$ (and consequently, $l_p^3(j) > l_q^3(j)$). Label set $p$ is \emph{dominated} by label set $q$ at node $j$, if either: (a) $l_p^1(j) > l_q^1(j)$ and $l_p^2(j) = l_q^2(j)$ or (b) $l_p^1(j) = l_q^1(j)$ and $l_p^2(j) > l_q^2(j)$ (and again consequently, $l_p^3(j) > l_q^3(j)$).  A label set that is not strictly dominated or dominated, is \emph{non-dominated}. 
\end{definition}

\begin{remark} \label{remark:1}
Observe that each node $j$ can have at most $W$ non-dominated labels, where $W$ is the maximum number of discrete values of weights $w_{ij}$, over all edges $(i,j)$ in the network. This is a consequence of the fact $l_p^1(j)$ and $l_p^2(j)$ are the only two quantities being tracked (and $l_p^3(j)$ is the product of $l_p^1(j)$ and $l_p^2(j)$); and the possibilities for non-dominated labels are if  $l_p^1(j) < l_q^1(j)$ and $l_p^2(j) > l_q^2(j)$ (or vice versa) in labels $p$ and $q$ of node $j$. Also note that all non-dominated labels must be maintained since the optimal substructure property does not hold,  and each such label must be propagated to the downstream nodes because it may dominate after propagation. Since $l_p^1$ is always an integer, the number of discrete weights upper bounds the possible number of non-dominated labels that can exist at each node.
\end{remark}

\begin{lemma}\label{lemma_monotonic}
Along a given $a-b$ path, the value of the labels on the path monotonically (but not necessarily strictly monotonically) increases.
\end{lemma}
\begin{proof}
Note that the values in labels $l_p^1$ and $l_p^2$ are always non-negative because the labels measure the number of edges and the max width thus far, respectively. The UpdateLabels operation defined below can only increase label values, and thus, on a given $a-b$ path (including if $a$ is the same as $b$), the label values only increase. 
\end{proof}

\begin{lemma}\label{lemma_cycles}
The $s-t$ path corresponding to bottleneck distance will not contain any cycles. 
\end{lemma}
\begin{proof}
Proof is by contradiction. Suppose the $s-t$ path with the bottleneck distance contains a cycle, that is the path is $s -\ldots-u-u_1 \ldots u_k-u-\ldots-t$, with cycle $u-u_1 \ldots u_k-u$. However, by Lemma \ref{lemma_monotonic}, we know that as we travel along $u-u_1\ldots u_k-u$, the value of $d_B(s,t)$ only increases. Therefore the path $s - \ldots -u- \ldots -t$ that excludes the cycle will have a smaller value of $d_B(s,t)$, contradicting that $s - \ldots -u-u_1  \ldots u_k-u-\ldots-t$ is the shortest path.
\end{proof}

\begin{theorem}
When Algorithm \ref{bottleneck_dijkstra_one_all} terminates, all nodes $i$ have the bottleneck distance from $s$ to $i$ as labels.
\end{theorem}
\begin{proof}
The correctness of the algorithm follows because the label sets keep track of $l_p^1(i), l_p^2(i)$, and $l_p^3(i)$ for each node $i$ and each label $p$ corresponding to that node $i$. At each iteration, each non-dominated label is explored, in increasing order of $l_p^3(i)$. That is, each label set from each edge out of the selected node $i$ is propagated downstream during the UpdateLabels step. Due to the ConsolidateLabels step at each node in each iteration, a maximum of $W$ labels can be present at each node (Remark \ref{remark:1}). All non-dominated labels are explored before termination, and due to Lemma \ref{lemma_monotonic} and \ref{lemma_cycles}, all possible paths are explored, resulting in the short-and-wide path. 
\end{proof}


\begin{algorithm}
\caption{One-to-All Bottleneck Distance}\label{bottleneck_dijkstra_one_all}
\begin{algorithmic}[1]
\Initialize {$S = \emptyset, S' = \emptyset$\\
 $l_1^k (s)= 0, \: \forall \: k = 1, 2, 3$ \\
$S = S \cup \{s\}, S' = S' \cup \{l_p^k(s)\} \: \forall \: k = 1, 2, 3, p = 1, 2$ \\ $l_p^k (j) = \infty \:; \forall \: k = 1, 2, 3, p = 1,2; np(j) = 0, p_p(j) = -1$ and $pp(j) = -1 \; \forall \: j \in N\setminus\{s\}$}\\
$V' = V' \cup \{l_p^k(j)\} \: \forall j \in V, \: \forall \: k = 1, 2, 3$ 
\While {$S \neq V$ and $S' \neq V'$} 
	\For {each node $j \in V \setminus S$}
		\For {each non-dominated label set $L_p(j)$ (see Definition \ref{def:dominated_labels})}
		\State $T := T \cup \{L_p(j)\}$ 
		\EndFor
	\EndFor
	\State Order label sets in $T$ in increasing order of $l_p^3(j), \forall \: p, \forall \: j$, breaking ties arbitrarily. Find label set $L_p(i) \in T$ such that $l_p^3(i) = \min_{L \in V'\setminus S'} l_q^3(j)$
	\State $T = T\setminus\{L_p(i)\}$, $S' = S \cup \{L_p(i)\}$, $S := S \cup \{i\}$
	\State \textbf{UpdateLabels($L_p(i)$):} 
		\For {each edge $(i,j) \in E$} 
			\For {each label set $q = 1,..,np(j)$ at node $j$}
				\If {$l_q^1(j) > l_p^1(i) + 1 \wedge l_q^2(j) > max\{l_p^2(i), w_{ij}\}$}
					\State \hspace{\algorithmicindent}$pred_{q}(j) = i, pp_{q}(j) = p$
					\State \hspace{\algorithmicindent}$l_{q}^1(j) = l_{p}^{1}(i) + 1$
					\State \hspace{\algorithmicindent}$l_q^2(j) = \max \{l_p^2(i), w_{ij}\}$ 	
					\State \hspace{\algorithmicindent}$l_q^3(j) = l_q^1(j) * l_q^2(j)$ 		
				\ElsIf {$(l_q^1(j) \leq l_p^1(i) + 1 $ and $l_q^2(j) > max\{l_p^2(i), w_{ij}\}) \vee (l_q^1(j) > l_p^1(i) + 1 $ and $l_q^2(j) \leq max\{l_p^2(i), w_{ij}\})$}
					\State Create new temporary label $L_{q'}(j)$ at $j$ with $pred_{q'}(j) = i, pp_{q'}(j) = p, l_{q'}^1 = l_p^1(i) + 1, l_{q'}^2(j) = max\{l_p^2(i), w_{ij}\}, l_q^3(j) = l_{q'}^1(j) * l_{q'}^2(j)$
				\EndIf
			\EndFor
			\State \textbf{ConsolidateLabels(j):}
			\For {all (including temporary) labels  $p$ at $i$}
				\State Delete all dominated labels (Definition \ref{def:dominated_labels}). Also combine labels $p$ and $q$ with $l_q^1(j) = l_p^1(j)$ and  $l_q^2(j) = l_p^2(j)$. Temporary label made permanent if non-dominated. Update $np(j)$.
			\EndFor
		\EndFor
\EndWhile
\end{algorithmic}
\end{algorithm}

Having established correctness, we also analyze the computational complexity, in terms of the number of nodes $n$, number of edges $m$, and the number of possible discrete edge weights $W$.
\begin{lemma}
\label{lemma:one_all_runningtime}
The running time of Algorithm \ref{bottleneck_dijkstra_one_all} is $O(W^2 m\log n)$, which is pseudopolynomial.
\end{lemma} 
\begin{proof}
The algorithm is structured like Dijkstra's algorithm which has complexity $O(m \log n)$ \cite{AhujaMO1993}, but during each iteration, since the UpdateLabels step propagates from each label out of each edge and at each node, a consolidation step must be performed. With a discrete number of weights $W$ over all edges in the network, the possible labels at each node is upper bounded by $W$. Hence, the order of the algorithm is $O(W^2m\log n)$.
\end{proof}

\subsection{All-Pairs Bottleneck Distance Algorithm}
Now we consider the computation for all-pairs, rather than separately computing for each source-destination pair.  This is Algorithm \ref{bottleneck_floydwarshall}.

\begin{algorithm}[htbp]
\caption{All-to-All Bottleneck Distances}\label{bottleneck_floydwarshall}
\begin{algorithmic}[1]
\Initialize {$l^1[i,j] = \emptyset \: \forall \: i,j \in N$\\ 
			$l^2[i,j] = \emptyset \: \forall \: i,j \in N$\\
			Let $l_1^1[i,j] = 1$ if nodes $i$ and $j$ are adjacent\\
			Let $l_1^2[i,j]$ be the weight between nodes $i$ and $j$ if they are connected\\
			$np[i,j] = l_1^1[i,j] \: \forall \: i,j \in N$}
\For {all nodes $k \in N$}
	\For {all nodes $i \in N$}
		\For {all nodes $j \in N$}
			\State \textbf{Node Insertion} on label sets $L[i,k]$ and $L[k,j]$
			
			\State \textbf{Maximize Labels} on label sets $L[i,j]$ and $L'[i,j]$
		\EndFor
	\EndFor
\EndFor

\end{algorithmic}
\end{algorithm}

\begin{algorithm}[htbp]
	\caption{Node Insertion}\label{floydwarshall-insertion}
	\begin{algorithmic}[1]
\Initialize {$p = np[i,k]; q = np[k,j]$}

\While {$L_p[i,k]$ and $L_q[k,j]$ exist} 
	\State append to front $l_p^1[i,k] + l_q^1[k,j]$ to $l'^1[i,j]$
	\State append to front $\max \{l_p^2[i,k], l_q^2[k,j]\}$ to $l'^2[i,j]$
	
	\If {$l_p^2[i,k] = l_q^2[k,i]$}
		$p = p-1; q = q-1$
	\EndIf
	\If {$l_p^2[i,k] > l_q^2[k,i]$}
		$p = p-1$
	\EndIf
	\If {$l_p^2[i,k] < l_q^2[k,i]$}
		$q = q-1$
	\EndIf
\EndWhile
\end{algorithmic}
\end{algorithm}

\begin{algorithm}
	\caption{Maximize Labels}\label{floydwarshall-max}
	\begin{algorithmic}[1]
\Initialize {$bestL^1 = \infty$ \\
	$p,q = 1$\\ 
	$np[i,j] = 0$\\
Note $L[i,j]$ and $L'[i,j]$ are sorted ascending by $l^2$}

\While {$L_p[i,k]$ and $L_q'[i,j]$ exist} 
	\If {$l_p^2[i,j] < l_q'^2[i,j]$}
		\If {$l_p^1[i,j] < bestL^1$}
			\State append $L_p[i,j]$ to $L''[i,j]$
			\State $bestL^1 = l_p^1[i,j]$
		\EndIf
		\State $p = p + 1$
	\ElsIf {$l_p^2[i,j] > l_q'^2[i,j]$}
		\If {$l_q'^1[i,j] < bestL^1$}
			\State append $L_q'[i,j]$ to $L''[i,j]$
			\State $bestL^1 = l_q'^1[i,j]$
		\EndIf
		\State $q = q + 1$
	\ElsIf {$l_p^2[i,j] = l_q'^2[i,j]$}
		\If {$l_p^1[i,j] <= l_q'^1[i,j] \wedge l_p^1[i,j] < bestL^1$}
			\State append $L_p[i,j]$ to $L''[i,j]$
			\State $bestL^1 = l_p^1[i,j]$
		\ElsIf {$l_p^1[i,j] > l_q'^1[i,j] \wedge l_q'^1[i,j] < bestL^1$}
			\State append $L_q'[i,j]$ to $L''[i,j]$
			\State $bestL^1 = l_q'^1[i,j]$
		\EndIf
		\State $p = p + 1$
		\State $q = q + 1$
	\EndIf
\EndWhile
\For {each remaining node in $L[i,j]$ or $L'[i,j]$}
	\If {$l^1[i,j] < bestL^1$} append to $L''[i,j]$
	\EndIf
\EndFor
\State $L[i,j] = L''[i,j]$
\State Delete $L''[i,j]$ and $L'[i,j]$
\State update $np[i,j]$
	
	\end{algorithmic}
\end{algorithm}

\begin{theorem}\label{theorem_fw}
When Algorithm \ref{bottleneck_floydwarshall} terminates, all edges $(i,j)$ have  bottleneck distance from $i$ to $j$ as labels.
\end{theorem}
\begin{proof}
Note that the short-and-wide path between any nodes $i$ and $j$ will not contain any cycles, and therefore have at most $n-1$ edges. In the $k$th iteration, we consider adding node $k$ to each edge along the path connecting $i$ and $j$, from each label set at $i$ to each label set at $j$. This algorithm is equivalent to enumerating all possible combinations of the label sets $[i,k]$ and $[k,j]$. If the short-and-wide path between $i$ and $j$ contains node $k$, then the condition $d[i,j] > \max\{l_2[i,k],l_2[k,j]\} \times (l_1[i,k] + l_1[k,j])$ will be violated in iteration $k$, and $k$ is added to one of the edges on the path. If there is an improving or non-dominated label that can be added to a node, it is added even if $k$ is not on the bottleneck path between $i$ and $j$. If the bottleneck path does not contain node $k$, then the condition will not be violated and the current best path and distance are retained, leading to correctness.
\end{proof}

After establishing correctness, we characterize computational complexity in terms of the number of nodes $n$, number of edges $m$, and the number of discrete edge weights $W$.

\begin{theorem}\label{theorem_runtime_fw}		
	The worst-case runtime for Algorithm \ref{bottleneck_floydwarshall} is pseudopolynomial $O(n^3W^2)$.		
\end{theorem}
\begin{proof}		
Note that the structure of the algorithm is similar to the Floyd-Warshall algorithm for shortest paths, which is $O(n^3)$ \cite{AhujaMO1993}. Rather than one comparison inside of this, the bottleneck distance algorithm has two sub-algorithms, Algorithm \ref{floydwarshall-insertion} and Algorithm \ref{floydwarshall-max}, to compare all of the label sets. Each sub-algorithm at worst iterates through two label sets in parallel with size $W$ if there are $W$ levels of weights based on Remark \ref{remark:1} (or maximum size $m$). Therefore their runtimes are $O(W^2)$ (or $O(m^2)$ if all edges have different weights). Combining this with the whole algorithm gives the total runtime as $O(n^3W^2)$ (or worst-case of $O(n^3m^2)$).
\end{proof}		

\section{All-Pairs Bottleneck Distance Distribution in Complex Networks}

With efficient algorithms in place, we can compute the all-pairs bottleneck distance distribution for several undirected,  weighted, real-world networks drawn from the Index of Complex Networks \cite{ICON} database, which is publicly available. We choose networks that span a variety of systems, including transportation networks, biological networks, and social networks that may naturally support pipelined flow. Table \ref{tab:networks_characteristics} details the size, type, and sources of each network.  

\begin{table}[htbp]
\caption{Characteristics of Real-World Networks}
    \label{tab:networks_characteristics}
\scriptsize
    \centering
    \begin{tabular}{l l l l l}
    \hline
       Name  & Nodes & Edges & Type & Source\\
    \hline   
       US airports & 1574 & 28236 & Transportation & US airport networks (2010) \cite{ICON}\\
       Mumbai bus routes & 2266 & 3042 & Transportation &India bus routes (2016) \cite{ICON}\\
       Chennai bus routes & 1009 & 1610 & Transportation & India bus routes (2016) \cite{ICON}\\
       Author collaborations & 475 & 625 & Social & Social Networks authors (2008) \cite{ICON}\\
       Free-ranging dogs & 108 & 1296 & Social & Wilson-Aggarwal dogs \cite{ICON}\\
       Game of Thrones & 107 & 353 & Social & Game of Thrones coappearances \cite{ICON}\\
       Resting state fMRI network & 638 & 18625 & Biological & Human brain functional coactivations \cite{ICON}\\
       Human brain coactivation & 638 & 18625 & Biological & Human brain functional coactivations \cite{ICON}\\
    \hline
    \end{tabular}
\end{table}

In Figure \ref{fig:empirical_survival_graphs} we plot the survival functions for the all-to-all  geodesic, weighted, and bottleneck distances $d_G$, $d_W$, and $d_B$. For each network, we find that the inequalities $d_W \leq d_B \leq d_G$ hold, as required by Propositions \ref{prop:order1} and \ref{prop:order2}.  If all edge weights were $1$, all distance metrics are equivalent, i.e., $d_W = d_B = d_G$. For example, note that in the Authors' collaboration network (Figure \ref{fig:authors_collaboration}), the values of $d_B$ and $d_G$ very nearly coincide because the maximum weight along nearly every path is $1$. However, because a significant number of weights are far from $1$, the values of $d_W$ and $d_B$ do not coincide. Additionally, we observe that the geodesic, bottleneck, and weighted distances diverge as the weights are distributed away from $1$.

\begin{figure}[htbp]
    \centering
    \begin{subfigure}[b]{0.38\textwidth}
        \centering
        \includegraphics[width=\textwidth]{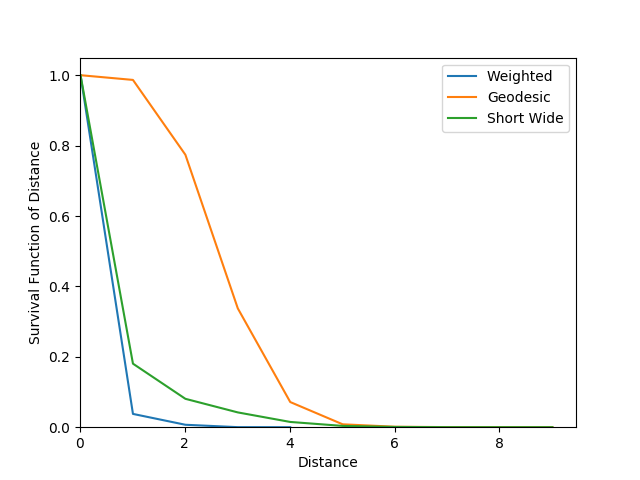}
        \caption{US airports 2010}\label{fig:airport}
    \end{subfigure}%
    \hspace{2.2cm}
    \begin{subfigure}[b]{0.38\textwidth}
        \centering
        \includegraphics[width=\textwidth]{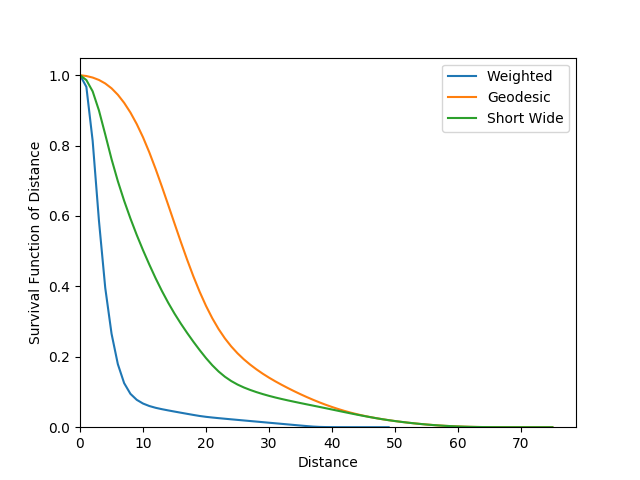}
        \caption{Mumbai bus network}
        \label{fig:Mumbai_bus_network}
    \end{subfigure}
    \hspace{2.2cm}
    \begin{subfigure}[b]{0.38\textwidth}
        \centering
        \includegraphics[width=\textwidth]{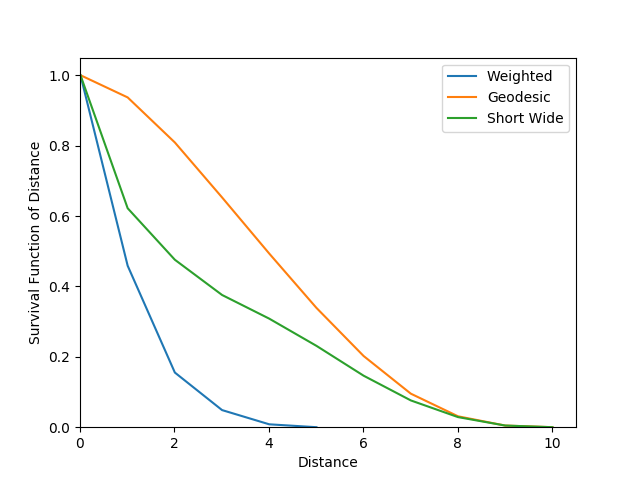}
        \caption{Chennai bus network}
   \label{fig:Chennai_bus_network}
   \end{subfigure}
   \hspace{2.2cm}
    \begin{subfigure}[b]{0.38\textwidth}
        \centering
        \includegraphics[width=\textwidth]{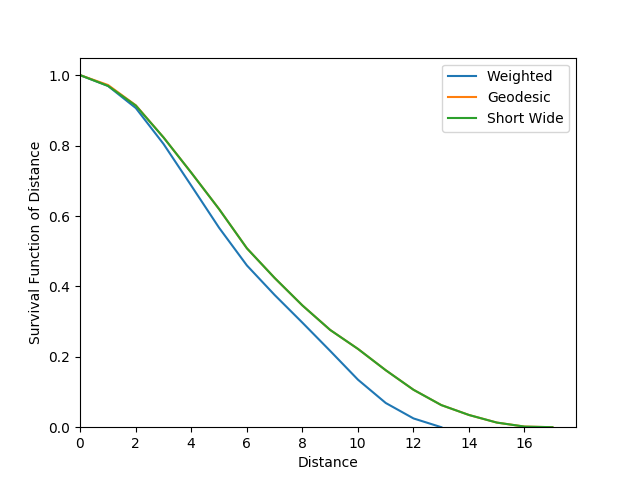}
        \caption{Authors' collaboration network}
        \label{fig:authors_collaboration}
    \end{subfigure}
    \hspace{1.5cm}
    \begin{subfigure}[b]{0.38\textwidth}
        \centering
        \includegraphics[width=\textwidth]{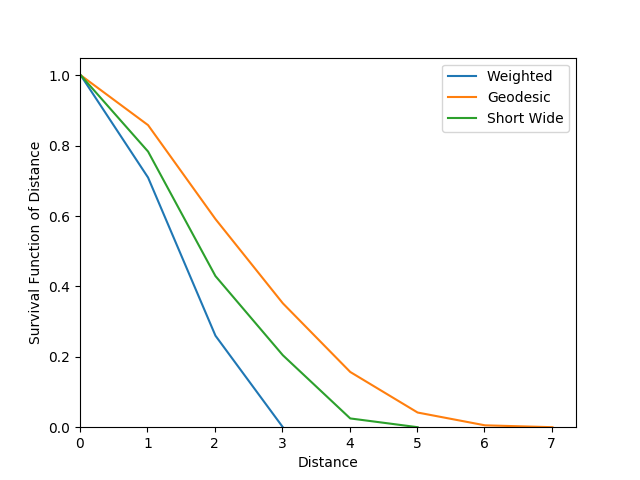}
        \caption{Free ranging dogs social network}
        \label{fig:dogs_social_network_Kakale}
    \end{subfigure}
    \hspace{1.5cm}
    \begin{subfigure}[b]{0.38\textwidth}
        \centering
        \includegraphics[width=\textwidth]{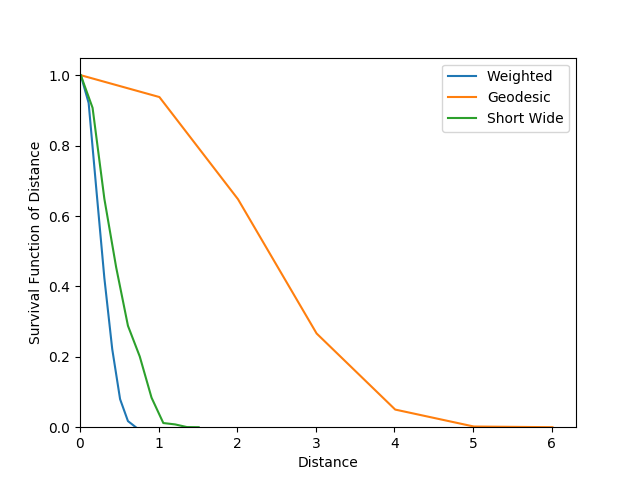}
        \caption{Game of Thrones coappearances}
        \label{fig:GoT_coappearances}
    \end{subfigure}
    \hspace{1.5cm}
    \begin{subfigure}[b]{0.38\textwidth}
        \centering
        \includegraphics[width=\textwidth]{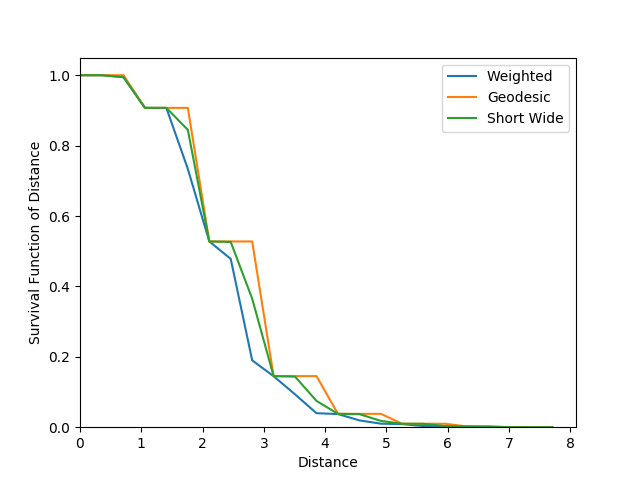}
        \caption{Resting State fMRI network}
    \end{subfigure}
    \hspace{1.5cm}
    \begin{subfigure}[b]{0.38\textwidth}
        \centering
        \includegraphics[width=\textwidth]{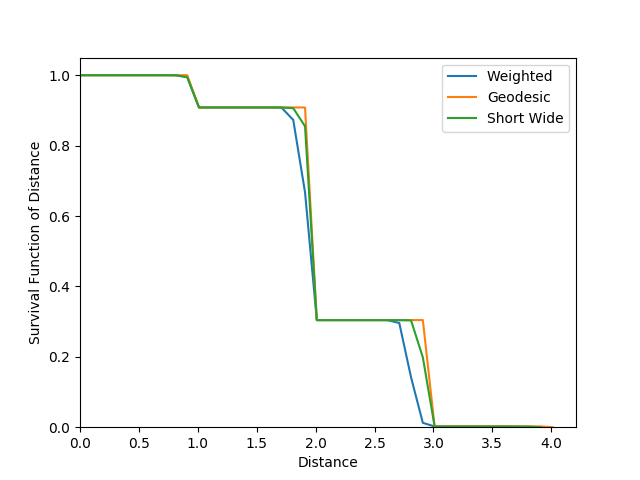}
        \caption{Human brain coactivation network}
    \end{subfigure}
    \caption{Survival function for the empirical all-pairs lengths.}\label{fig:empirical_survival_graphs}
\end{figure}

Previous investigations of the geodesic distance distribution of random graphs had suggested good fits by Weibull, gamma, lognormal, and generalized three-parameter gamma distributions \cite{BauckhageKH2015} as well as basic generative models to explain these distributions.  We consider the same parametric families to understand if short-and-wide path length distributions are well-described by such parametric forms. For each network in Table \ref{tab:networks_characteristics}, we fit the survival functions from Figure \ref{fig:empirical_survival_graphs} and find goodness-of-fit for several different parametric families using the Kolmogorov-Smirnov test.  No single parametric family was best for all networks, but the gamma distribution provided reasonable fits for several networks, as shown in Table \ref{tab:networks_fit_values} which provides the fitted distribution parameters, along with the $\chi^2$ and $p$-values from the test.  As far as we can tell, there is no universality in the bottleneck distance distribution across networks. 

\begin{table}[htbp]
\caption{Gamma Distribution Fits of Bottleneck Distance}
    \label{tab:networks_fit_values}
\scriptsize
    \centering
    \begin{tabular}{l l l l l l}
    \hline
       Name  &  Shape & Location & Scale & $\chi^2$ & $p$-value\\
    \hline   
       US airports  & 0.27 & 0 & 0.64 & 5.80 & 0.63 \\
       Chennai bus routes  & 0.54 & 0 & 0.22 & 1.03 & 0.06 \\
       Mumbai bus routes  & 0.41 & 0 & 0.23 & 21.25 &  0.095 \\
       Author collaborations  & 0.83 & 0 & 0.32 & 8.19 & 0.22 \\
       Free-ranging dogs  & 0.67 & 0 & 0.40 & 3.56 & 0.61 \\
       Game of Thrones  & 0.55 & 0 & 0.18 & 4.71 & 0.008 \\
       Resting state fMRI network  & 0.61 & 0 & 0.19 & 4.91 & 0.23 \\
       Human brain coactivation &  368.63 & -7.47 & 0.02 & 8.756 & 0.85 \\
    \hline
    \end{tabular}
\end{table}

\section{\emph{C.\ elegans} Neuronal Network: Global and Local Flow}

We turn attention specifically to the  \emph{C.\ elegans} gap junction network, to investigate how information flow limits behavior.  A bound on the Shannon capacity [bits/sec] of a single gap junction is developed in the Appendix.  Although there is no reason to believe capacity-achieving codes are used in neural signaling, Shannon capacity provides bounds on the information rate for any signaling scheme.  

The topology of the gap junction network
has been characterized in some detail in our prior work \cite{VarshneyCPHC2011}.  The somatic network consists 
of a giant component comprising $248$ neurons, two small connected components, and several
isolated neurons.  Within the giant component, the average geodesic distance between two neurons
is $4.52$.  Since this characteristic path length is similar to that of a random graph and
since the clustering coefficient is large with respect to a random graph, the network is 
said to be a small-world network.  Moreover, the \emph{C.\ elegans} network overall 
has good expander properties \cite[App.~C]{VarshneyCPHC2010_arxiv}.

Here we compute the effective diameter of the giant component 
of the gap junction network, with respect to the bottleneck distance.  We also find upper and lower bounds.
Figure~\ref{fig:distance_survival} shows the survival function of the empirical all-pairs 
geodesic, weighted, and bottleneck distances.
As can be observed, the effective diameter for bottleneck distance is between $6$ and $7$.  The upper and lower bounds are close to one another since, as shown in Figure~\ref{fig:bottle_width},
the minimax width of paths in the \emph{C.\ elegans} network when ignoring path length is almost always one 
inverse gap junction rather than smaller.  Taking path lengths into account only increases the bottleneck width.

\begin{figure}
  \centering
  \includegraphics[width=3in]{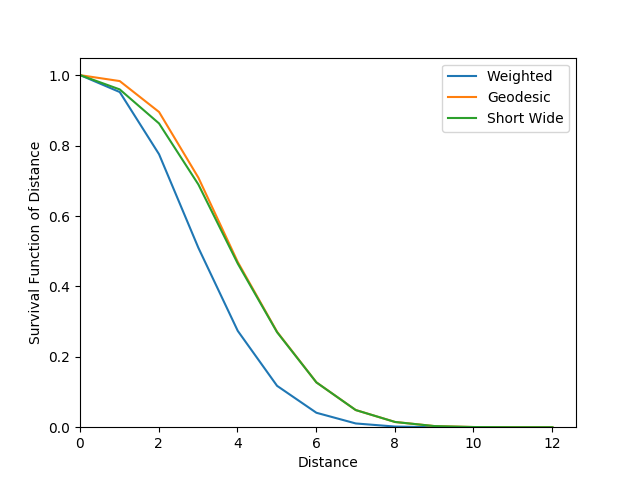}
  \caption{Survival function for the empirical all-pairs distance distributions of the \emph{C.\ elegans} gap junction neuronal 
	network giant component.  The weighted and bottleneck distances are listed in terms of inverse gap junctions.
	}
  \label{fig:distance_survival}
\end{figure}

\begin{figure}
  \centering
  \includegraphics[width=3in]{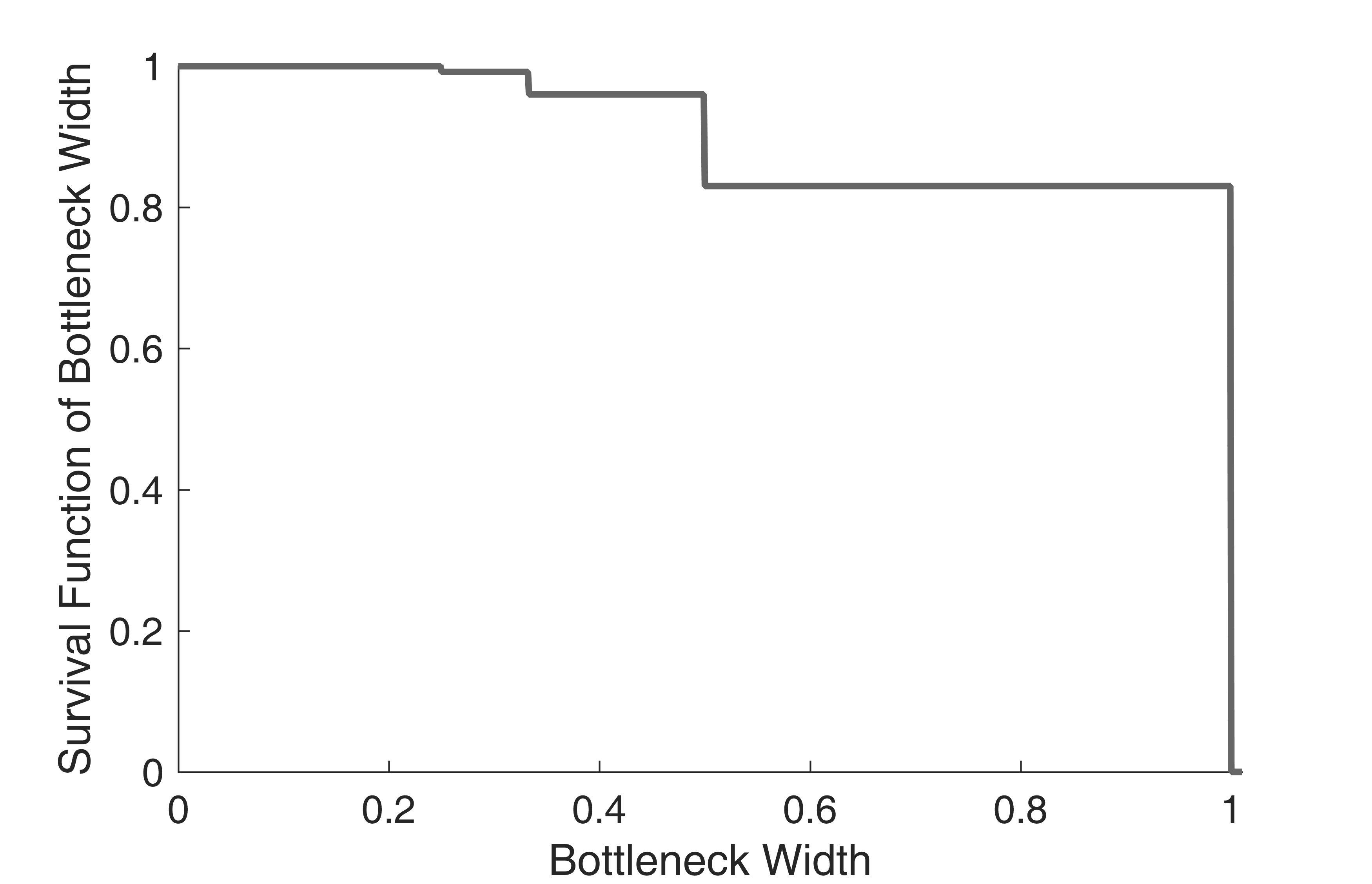}
  \caption{Survival function for the empirical all-pairs bottleneck width of the \emph{C.\ elegans} gap junction neuronal 
	network giant component, without considering path length.  The width is listed in terms of inverse gap junctions.}
  \label{fig:bottle_width}
\end{figure}

\subsection{Limits in Computation Speed}

Having characterized the channel capacity of links and 
the topology of neuronal connectivity, we now develop an information-theoretic model of computation, which in turn yields a limit on computation speed derived from information flow bounds.

Consider the chemosensing problem faced by an organism like \emph{C.\ elegans}.
It has $700$ different types of chemoreceptors \cite{WhittakerS2004} and must take 
behavioral actions based on the chemical properties of its environment 
\cite{ChalasaniCTGRGB2007,MacoskoPFCBCB2009}.  Suppose that differentiating between $700$
chemicals requires $10$ bits of information, which we call the message volume $\log M$.
To perform an action, the neurons must reach consensus among themselves based on the sensory
neuron signals \cite{MacoskoPFCBCB2009}.  This may, in 
principle, need transport of information between all parts of the neuronal 
network. We make the following natural assumption: the consensus
time is bounded by the amount of time it takes to transport $10$ bits of
information across the network (maximum over any two pair of neurons). 

The intuition behind this assumption is the following, which follows from the sparsity of
sensory response in \emph{C.~elegans} \cite{ZaslaverLSGYS2015}. Suppose one part of the 
neuronal network has ``strong'' sensing information about an event in
the environment, e.g.\ worm near a chemical,  but the 
other part has little or no such sensing information. Then effectively
all information is communicated from one part of the network to
the other. Accounting for such instances naturally justifies the
above assumption. 
Note that the actual computational procedures used by \emph{C.\ elegans} may require
several sweeps of signals through the organism, but for bounding purposes, we assume
that one sweep is enough to spread the requisite information.

In effect, the time for transportation of information
across the network and hence to reach consensus is bounded below by 
\begin{align*}
t = \frac{D_e \log M}{C} = \frac{7 \cdot 10}{1700} = 0.041 \mbox{s.}
\end{align*}
This bound uses the geodesic effective diameter and assumes bandwidth $1700$ Hz or equivalently
a refractory period of $1/1.7$ ms (see Appendix).  There is evidence
suggesting that the \emph{C.~elegans} refractory period is likely to be near $1$ ms instead.
In that case, the above bound of $41$ ms would become $70$ ms. 
The bound of $41$ ms or $70$ ms applies to the whole giant component. 

There may be smaller subcircuits within the neuronal network responsible
for specific functional reactions, within the giant component.
If consensus is to be reached only in those functional subcircuits, we
should utilize their diameter in place of effective diameter $7$. As explained in
Section~\ref{ssec:hubspoke}, they may have extremal 
diameters of $2$. Then the information propagation time
would be bounded by $12$ ms (or $20$ ms under $1$ ms refractory
period).

Using circuit-theoretic techniques, the predicted timescale of operation 
of functional circuits was between $20$ ms and $83$ ms \cite{VarshneyCPHC2011}. 
This is clearly an excellent match to the range predicted by our 
bounds: $12$ ms to $41$ ms (or $20$ ms to $70$ ms under the $1$ ms refractory period). This is rather surprising
given that our technique is only attempting to derive 
fundamental lower bounds using anatomical information. 

Now we compare our lower-bound predictions to experimentally observed behavioral times. 
This is a true test of our methods in bounding the propagation time for decision-making information.  The 
behavioral switch times in response to chemical gradients as fast as 
$200$ ms have been observed in the literature \cite[Supp.~Fig.~1]{AlbrechtB2011}.
Since the time required for motor action like turning around must also be taken account---the 
worm can straighten itself in viscous fluids within $6$ to $20$ ms \cite[Fig.~4B]{Fang-YenWXKKCWS2010}---the 
lower-bound results are in agreement.  

Collectively, these agreements with the model calculation bounds 
suggest that information propagation is
likely to be a primary bottleneck in the behavioral decision making
of \emph{C.\ elegans}. 

\subsection{Hub-and-Spoke Architecture}
\label{ssec:hubspoke}

As mentioned earlier, there are smaller subcircuits of the 
neuronal network that are responsible for certain functional
reactions. For such subcircuits, the hub-and-spoke architecture
has optimality properties.  
The basic premise is that the diameter of 
such a subcircuit should be small for computational speed; a hub-and-spoke
network structure provides the smallest possible diameter
of $2$ with the constraint on the number of edges as well
as connectivity requirement.  Formally, we state the following
easy fact. 
\begin{proposition}
Given a connected graph $G$ of $n \geq 3$ nodes and $n-1$ edges, 
the smallest possible diameter is $2$ and is achieved by the 
hub-and-spoke structure. 
\end{proposition}
\begin{proof}
Clearly for a connected network with $n \geq 3$ nodes and $n-1$ edges, 
it is a tree (no cycles). Further if there are $n \geq 3$ nodes and $n-1$
edges, there must be a pair of nodes not connected to each other
through an edge. Since the graph is connected, they must be at least
$2$ hops apart. That is, the diameter of such a graph must be at least 
$2$. A hub-and-spoke network, by construction has diameter $2$. 
\end{proof}

Indeed, as shown in Figure~\ref{fig:hubspoke}, certain known functional
subcircuits in the \emph{C.\ elegans} neuronal network do indeed follow
the hub-and-spoke architecture (or nearly so) \cite{MacoskoPFCBCB2009,RabinowitchCS2013}.
Note that other arguments also suggest the benefits to neuronal networks 
of small diameter \cite{KaiserH2006}, but not the optimality
of hub-and-spoke architectures.

\begin{figure}
  \centering
  \includegraphics[width=2.052in]{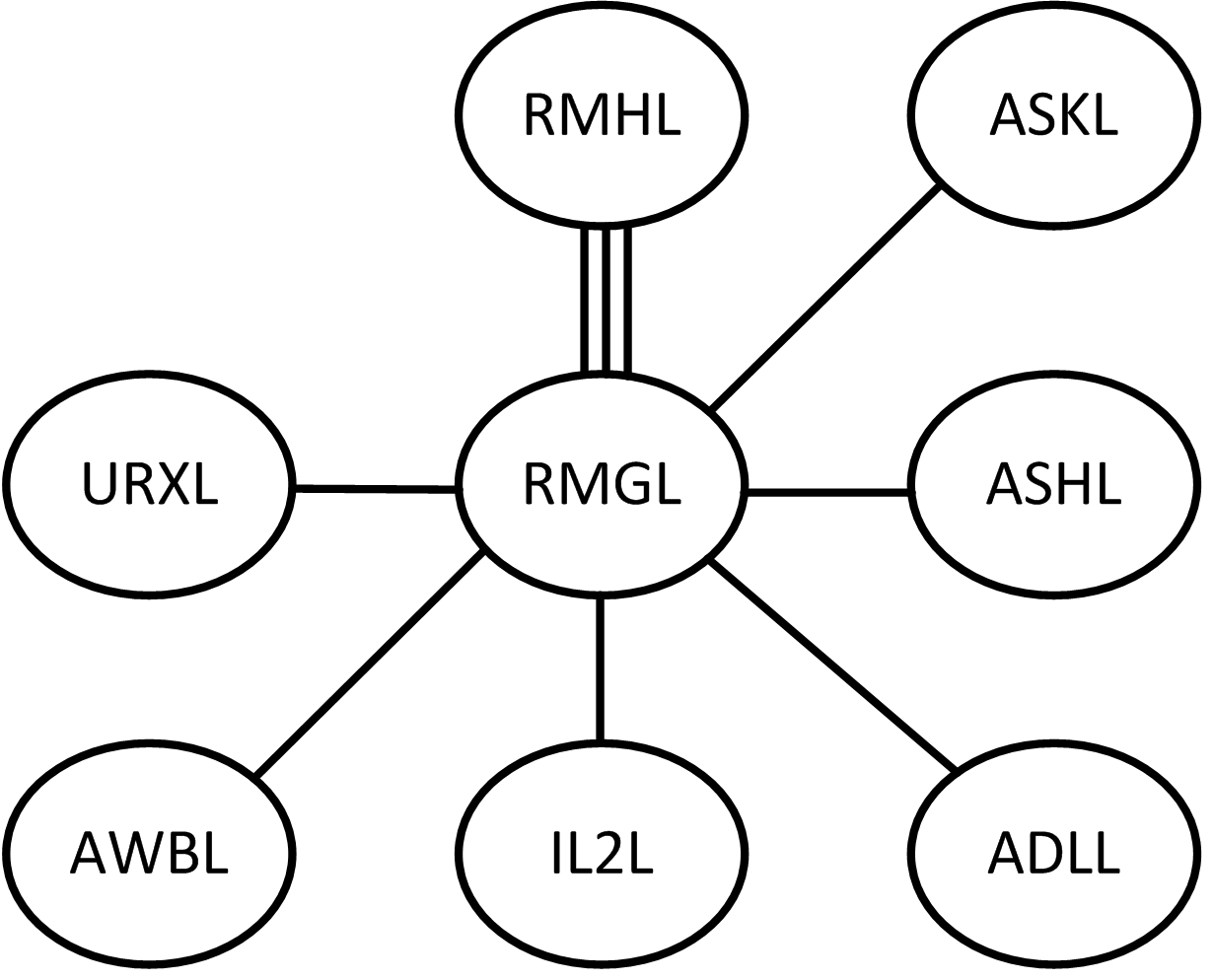} \qquad \includegraphics[width=2.052in]{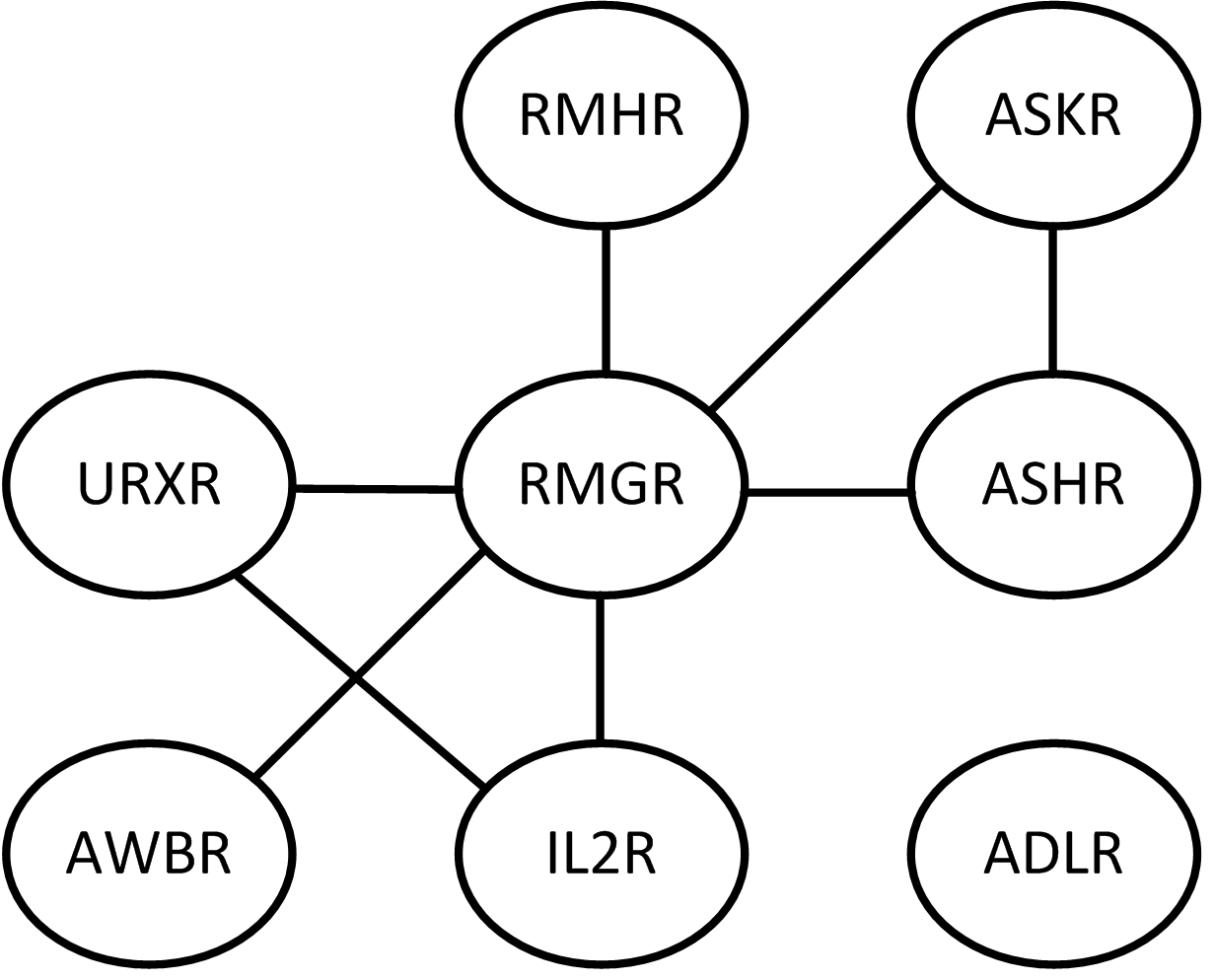} \\ (a) \qquad\qquad\qquad\qquad\qquad\qquad\qquad\qquad\quad (b)
  \caption{Gap junction connectivity of a known chemosensory subcircuit \cite{MacoskoPFCBCB2009}.  The worm is almost
	left-right symmetric; the circuits on both sides are shown. The circuit on the left side of the worm (a) is a 
	hub-and-spoke circuit whereas the circuit on the right side (b) is nearly so.}
  \label{fig:hubspoke}
\end{figure}

\subsection{Comparison to Random Networks}
Now we wish to study whether the bottleneck diameter of the \emph{C.~elegans} gap junction network is more than, less than, 
or similar to the bottleneck diameter of random graphs that have certain other network functionals fixed.

To evaluate the nonrandomness of the bottleneck diameter of the \emph{C.~elegans} network giant component, we compare it 
with the same quantity expected in random networks.  We start with a weighted version of the Erd\"{o}s-R\'{e}nyi random network 
ensemble because it is a basic ensemble.  Constructing the topology requires a single parameter, the probability of a connection 
between two neurons. There are $514$ gap junction connections over $279$ somatic neurons in \emph{C.~elegans}, and so we 
choose the probability of connection as $0.0133 = 2\times 514/279/278$.  After fixing the topology, we choose the multiplicity
of the connections by sampling randomly according to the \emph{C.~elegans} multiplicity distribution \cite[Fig.~3(B)]{VarshneyCPHC2011},
which is well-modeled as a power-law with parameter $2.76$.  Note that in general the giant component for such a construction will
be much larger than that of \emph{C.~elegans}.

Figure \ref{fig:distance_survivalER} shows the survival function of the empirical all-pairs geodesic, weighted, and bottleneck distances of one hundred random networks.  A random example is highlighted.  As can be observed,
the effective diameter bottleneck distance is roughly $6$, significantly less than that for the \emph{C.~elegans} network.

Now we consider a degree-matched weighted ensemble of random networks.  In such a random network, the degree distribution 
matches the degree distribution of the gap junction network; the degree of a neuron is the number of neurons with which it 
makes a gap junction. Such a random ensemble is created using a numerical rewiring procedure to generate samples \cite{MaslovS2002,ReiglAC2004}.
Upon fixing the topology, the multiplicity of connections is sampled as for the Erd\"{o}s-R\'{e}nyi ensemble.
Note that in general the giant component for such a construction will be much larger than that of \emph{C.~elegans}.

Figure \ref{fig:distance_survivalDM} shows the survival function of the empirical all-pairs geodesic, weighted, and bottleneck distances of one hundred random networks.  A random example is highlighted.  As can be observed,
the effective diameter for bottleneck distance is just below $5$, quite significantly less than that for the \emph{C.~elegans} network.
Comparing Figures \ref{fig:distance_survivalER} and \ref{fig:distance_survivalDM} to Figure~\ref{fig:distance_survival}, note that 
the results hold for many defining thresholds for effective diameter, not just $0.95$.



\begin{figure}
  \centering
  \includegraphics[width=3in]{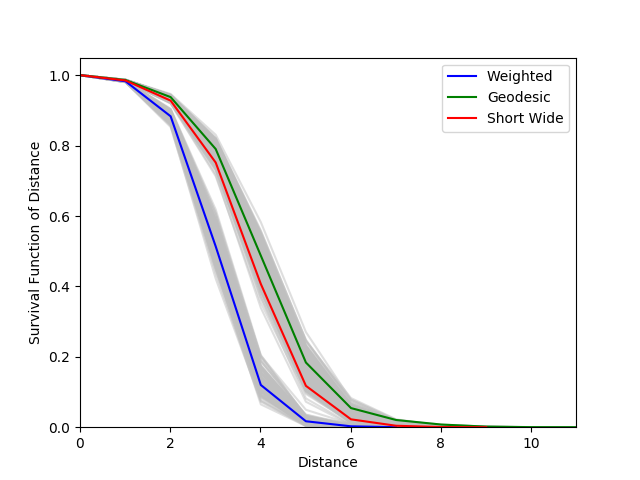}
  \caption{Survival function for the empirical all-pairs distance distributions of $100$ Erd\"{o}s-R\'{e}nyi random network giant components;
	a random example is highlighted in colored lines.  
	The weighted and bottleneck distances are in terms of inverse gap junctions.}
  \label{fig:distance_survivalER}
\end{figure}

\begin{figure}
  \centering
  \includegraphics[width=3in]{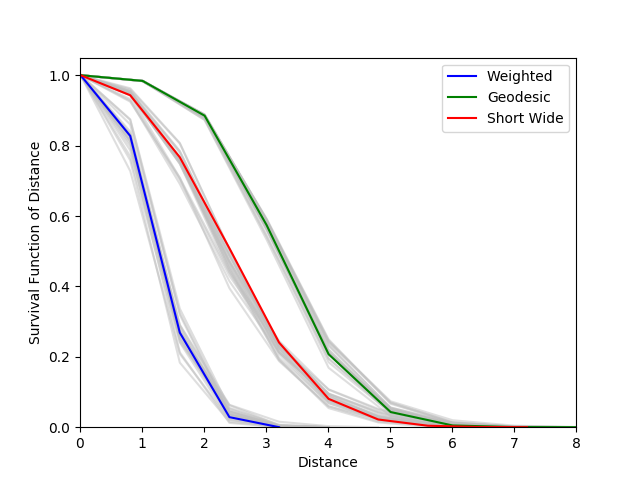}
  \caption{Survival function for the empirical all-pairs distance distributions of $100$ degree-matched random network giant components;
	a random example is highlighted in colored lines.
	The weighted and bottleneck distances are in terms of inverse gap junctions.}
  \label{fig:distance_survivalDM}
\end{figure}

These results reveal a key nonrandom feature in synaptic connectivity of the \emph{C.~elegans} gap junction network,
but perhaps contrary to expectation.  The network has a non-randomly \emph{worse} bottleneck diameter compared
to basic random graph ensembles.  It enables globally \emph{slower} behavioral speed than similar random networks. 
In contrast, Section~\ref{ssec:hubspoke} had found that at the micro-level of small functional sub-circuits, the \emph{C.~elegans} gap junction
network has several hub-and-spoke structures \cite{MacoskoPFCBCB2009,RabinowitchCS2013}, which are actually optimal 
from an information flow perspective.  Thus, these results lend greater nuance to efficient flow hypotheses in neuroscience.

\section{Conclusion}
We have modeled pipelined network flow, as arises in several biological, social, and technological networks and developed algorithms to efficiently find short-and-wide paths that optimize this novel network flow.  This new model of flow also specifically provides a new approach  to understand the limits of information propagation and behavioral speed supported by neuronal networks. Within the field of network information theory itself, it is of interest to study the novel notion of bottleneck distance in
detail from a theoretical perspective.  Since we did not universal scaling laws across several real-world networks, it is also of interest to analytically characterize its distribution in random ensembles such as Watts-Strogatz small worlds, Barab\'{a}si-Albert scale-free 
networks, Kronecker random graphs, or random geometric graphs.

Beyond our general study, we also specifically considered circuit neuroscience and connectomics, where the overarching goal is try to understand how an animal's 
behavior arises from the operations of its nervous system.  The nervous system
must transport information from one part to another, whether engaged in communication, computation, control, or maintenance.
This paper proposes a way to characterize information flow through the nervous system from 
detailed properties of anatomical connectivity data and to use this characterization to make lower bound statements on
the behavioral timescales of animals.

The efficacy of these techniques was demonstrated by explaining 
the communication bottlenecks in the gap junction network of the nematode \emph{C.\ elegans}.
Remarkably, the timescale lower bounds are predictive of behavioral timescales.

In considering the possibility of changing the network topology itself, we discovered that
the network has much worse bottleneck distance than similar random graphs (whether Erd\"{o}s-R\'{e}nyi or degree-matched). 
The network does not seem to be optimized for global information flow.
On the other hand, we noted the prominence of hub-and-spoke functional subcircuits in the \emph{C.~elegans} gap junction 
network and proved their optimality for information flow under number
of gap junctions constraints.
In terms of neural organization, this suggests that smaller subcircuits within the larger neuronal network
are responsible for specific functions, and these should have fast information flow (to quickly achieve the 
computational objective of that circuit, such as chemotaxis).  Behavioral speed of the global network may not be
biologically relevant.  

As more and more connectomes are uncovered and more details of the biophysical properties of 
synapses are determined, these information flow techniques may provide a general
methodology to understand how physical constraints lead to informational and thereby behavioral limits
in nervous systems.

\vskip6pt

\enlargethispage{20pt}

\subsubsection*{Authors' Contributions}
LM developed and implemented efficient algorithms for computing bottleneck distance, proved their correctness and characterized their complexity, participated in the design of the study, and carried out analysis; LRV conceived of the study, designed the study, coordinated the study, proved certain results, participated in data analysis, and drafted the manuscript; DS conceived of the pipelining model of network flow and proved certain theoretical results; NAP and MEG implemented algorithms and participated in data analysis. All authors gave final approval for publication and agree to be held accountable for the work performed therein.

\bibliographystyle{IEEEtran}
\bibliography{abrv,conf_abrv,neural}

\section*{Appendix: Gap Junction Links}
Although nearly all information-theoretic investigation of synaptic transmission has focused on electrochemical 
synapses \cite{RiekeWRB1997,BorstT1999}, purely electrical gap junctions are also rather ubiquitous, not only 
in \emph{C.\ elegans} but also in the mammalian brain \cite{ConnorsL2004}.
Among the $282$ somatic neurons in the \emph{C.\ elegans} connectome, there are $890$
gap junctions \cite{VarshneyCPHC2011}. 
As such, it is important to provide a mathematical model 
of signal flow through gap junctions, along with the noise that perturbs signals.

\subsection{Thermal Noise}
A gap junction is a hollow protein that allows electrical current to flow between cells, and thereby allows signaling in 
both directions.  In previous modeling efforts, it has been determined that gap junctions can essentially just be modeled as 
resistors with thermal noise that is additive white Gaussian (AWGN) \cite{FerreeL1999}.  Noise due to stochastic
chemical effects or due to random background synaptic activity \cite{ManwaniK1999} need not be considered.

The electrical conductance of a \emph{C.\ elegans} gap junction is $200$ pS \cite{VarshneyCPHC2011}; a resistor with 
conductance $200$ pS corresponds to a resistance of $5000$ M$\Omega$.  In order to compute the root mean square (RMS) 
thermal noise voltage $v_n$, we use the Johnson-Nyquist formula \cite{Johnson1928,Nyquist1928}:
\begin{equation}
v_n = \sqrt{4k_B T R \Delta f} \mbox{,}
\end{equation}
where $k_B$ is Boltzmann's constant $1.38 \times 10^{-23}$ J/K, $T$ is the absolute temperature in K, $R$ is the resistance in 
$\Omega$, and $\Delta f$ is the bandwidth in Hz over which the noise is measured.  We assume room temperature ($298$ K) 
and for reasons that will become evident in the sequel, we take the bandwidth to be $1700$ Hz.  Then
\begin{align}
v_n &= \sqrt{4 \cdot 1.38 \times 10^{-23} \cdot 298 \cdot 5 \times 10^{9} \cdot 1700} \\ \notag
&= 3.74 \times 10^{-4} \mbox{ V.} 
\end{align}
We consider the thermal noise RMS voltage of a \emph{C.\ elegans} gap junction to be this $0.374$ mV value.  

\subsection{Plateau Potentials}
Although one might hope that the signaling scheme that can be used over a gap junction is unrestricted,
there are biophysical constraints that impose limits.  We describe these signaling limitations
for \emph{C.\ elegans}.

Signaling arises from regenerative events in neurons.  
Electrophysiologists have described four main types of regenerative events: action potentials, graded 
potentials, intrinsic oscillations, and plateau potentials.  Plateau potentials are prolonged all-or-none 
depolarizations that can be triggered and terminated by brief positive- and negative-current pulses, respectively \cite{LockeryG2009}.
Plateau potentials are the biological equivalent of Schmitt triggers \cite{LockeryG2009}.

It is thought that plateau potentials may be used by many neurons in \emph{C.\ elegans},
and that they may arise through synaptic interaction \cite{LockeryG2009}.
RMD neurons in \emph{C.\ elegans} have two stable resting potentials, one near $-70$ mV and one near $-35$ mV \cite{LockeryG2009},
and these values are thought to hold across the nervous system.
These two levels can be thought of as the two possible input levels to a gap junction channel:
\begin{align}
v_0 &= -70 \times 10^{-3} \mbox{ V,}\\
v_1 &= -35 \times 10^{-3} \mbox{ V,}
\end{align}
as part of a random telegraph signal.

Switching between the two levels cannot happen arbitrarily quickly, as there is an absolute refractory
period between regenerative events due to biochemical constraints.  Unfortunately, the absolute refractory 
period for \emph{C.\ elegans} is not known due to the difficulty in performing the requisite electrophysiology 
experiments [email communication with S.~R.\ Lockery (14 Dec.\ 2010) and later confirmation].

The fastest impulse potentials observed in neurons in common reference are, perhaps, in the Renshaw cells of the mammalian spinal motor system, and have been reported as high as $1700$ per second \cite[p.~47]{AchacosoY1992}. Moving forward, we use this as a bound, however the typical value for the absolute refractory period across the animal kingdom is $1$ ms; \emph{C.\ elegans} may be even slower.  This is where the noise bandwidth value of $1700$ Hz also comes from.

\subsection{Capacity}
Having determined the noise distribution and the signaling constraints, we aim to find the capacity of this continuous-time, intertransition-time-restricted, binary-input AWGN channel.  This capacity computation problem has been studied by Chayat and Shamai \cite{ChayatS1999}, assuming no timing jitter.
Rather than using those precise results, we take a simplified approach through the discretization of time.
We assume slotted pulse amplitude modulation, which is robust to any timing jitter that may be present in 
\emph{C.\ elegans} signal propagation.  

In particular, we consider a discrete-time channel with 1700 channel usages per second,
two binary input levels of $-35$ and $-70$, and AWGN noise with standard deviation $0.374$.
As can be noted, the signal-to-noise ratio is rather high, $17.5^2 / 0.374^2 = 2.2 \times 10^3$, 
and so the capacity will be approximately one bit per channel usage.

Just to be sure, we compute this more precisely.  The capacity is 
\begin{equation}
C = h(Y) - \tfrac{1}{2} \log 2 \pi e \mbox{,}
\end{equation}
where $h(Y)$ is the differential entropy of the distribution:
\begin{align}
p(y) = \frac{1}{2}\left[\frac{1}{\sqrt{2\pi}}\exp\left\{\frac{-(y - \sqrt{\rm{SNR}})^2}{2}\right\}\right] 
 + \frac{1}{2}\left[\frac{1}{\sqrt{2\pi}}\exp\left\{\frac{-(y + \sqrt{\rm{SNR}})^2}{2}\right\}\right] \mbox{,}
\end{align}
with $\rm{SNR} =  2.2 \times 10^3$, see e.g.~\cite{VarshneySC2006}.
Performing the calculation demonstrates the rate loss below $1$ bit per channel usage due to noise is negligible. Thus we assume that the capacity of a \emph{C.\ elegans} gap junction is $1700$ bits per second,
or equivalently $5.9 \times 10^{-4}$ seconds per bit. 

Synaptic connection between two neurons may contain more than one gap junction.\footnote{The mean number of gap junctions between two connected neurons is $1.73$; 
see \cite[Fig.~3(b)]{VarshneyCPHC2011} for the distribution, which is well-modeled by a power law with exponent $2.76$.}  Although it is difficult to maintain electrical separation between individual gap junctions, for the purposes of this paper we assume that each gap junction can act independently.  Hence the channel capacity of parallel gap junction links is simply taken to be the number of gap junctions between the two neurons multiplied by the capacity of an individual gap junction.

\end{document}